\documentclass[accepted]{uai2022} 

\usepackage[american]{babel}

\usepackage{natbib} 
\usepackage{mathtools} 
\usepackage{booktabs} 
\usepackage{tikz} 

\usepackage[utf8]{inputenc} 
\usepackage[T1]{fontenc}    
\usepackage{url}            

\usepackage{booktabs}       
\usepackage{amsfonts}       
\usepackage{nicefrac}       
\usepackage{microtype}      
\usepackage{xcolor}         
\usepackage{amsthm}
\usepackage{amsmath}
\usepackage{bm}
\usepackage{graphicx}
\usepackage{caption}
\usepackage{subcaption}
\usepackage{thmtools}
\usepackage{thm-restate}
\usepackage{enumitem}
\usepackage{subcaption}
\usepackage{algorithm}
\usepackage[noend]{algorithmic}
\usepackage{mathtools}
\usepackage{xr-hyper}
\usepackage{hyperref}       
\usepackage{balance}       
\usepackage{etoolbox}
\robustify
\robustify\addtocounter
\robustify
\robustify

\newtheorem{proposition}{Proposition}
\theoremstyle{definition}

\title{Restless and Uncertain: Robust Policies for Restless Bandits \\ via Deep Multi-Agent Reinforcement Learning}

\author[1]{\href{mailto:<jkillian@g.harvard.edu>?Subject=Your UAI 2022 paper}{Jackson~A.~Killian}{}}
\author[1]{\href{mailto:<lily_xu@g.harvard.edu>?Subject=Your UAI 2022 paper}{Lily~Xu}{}}
\author[1,2]{\href{mailto:<arpitabiswas@seas.harvard.edu>?Subject=Your UAI 2022 paper}{Arpita~Biswas}{}}
\author[1,2]{\href{mailto:<milind_tambe@harvard.edu>?Subject=Your UAI 2022 paper}{Milind~Tambe}{}}
\affil[1]{%
    Computer Science,
    Harvard University,
    Cambridge, MA, USA
}
\affil[2]{%
    Center for Research on Computation and Society,
    Harvard University,
    Cambridge, MA, USA
}
  
\begin{document}
\maketitle

\begin{abstract}
We introduce robustness in \textit{restless multi-armed bandits} (RMABs), a popular model for constrained resource allocation among independent stochastic processes (arms). Nearly all RMAB techniques assume stochastic dynamics are precisely known. However, in many real-world settings, dynamics are estimated with significant \emph{uncertainty}, e.g., via historical data, which can lead to bad outcomes if ignored. To address this, we develop an algorithm to compute minimax regret--robust policies for RMABs.
Our approach uses a double oracle framework (oracles for \textit{agent} and \textit{nature}), which is often used for single-process robust planning but requires significant new techniques to accommodate the combinatorial nature of RMABs. Specifically, we design a deep reinforcement learning (RL) algorithm, DDLPO, which tackles the combinatorial challenge by learning an auxiliary ``$\lambda$-network'' in tandem with policy networks per arm, greatly reducing sample complexity, with guarantees on convergence. DDLPO, of general interest, implements our reward-maximizing agent oracle. We then tackle the challenging regret-maximizing nature oracle, a non-stationary RL challenge, by formulating it as a multi-agent RL problem between a policy optimizer and adversarial nature. 
This formulation is of general interest---we solve it for RMABs by creating a multi-agent extension of DDLPO with a shared critic. We show our approaches work well in three experimental domains.
\end{abstract}

\section{Introduction}\label{sec:intro}

Restless multi-armed bandits (RMABs), a model for constrained resource allocation among $N$ independent stochastic processes (arms), are widely studied.  Traditionally a \textit{binary-action} problem, in which a planner decides whether or not to act on each of $N$ arms, here we consider the \textit{multi-action} generalization \citep{killian2021multiAction,glazebrook2011general} which more accurately captures challenging real-world planning problems. Salient examples of RMABs include scheduling \citep{bagheri2015restless,yang2018optimal}, machine replacement \citep{glazebrook2006some,ruiz2020multi},
aerial vehicle routing \citep{le2008multi},
anti-poaching patrol planning \citep{qian2016restless}, and healthcare \citep{lee2019optimal,mate2020collapsing}.
While these works have established important theoretical foundations, they share one key limitation: assuming stochastic dynamics are precisely known. 
Having exact knowledge of dynamics is impossible in many real-world problems. For example, in healthcare intervention planning, the probability that a patient will adhere to treatment after receiving an intervention is not perfectly known \textit{a priori}; in anti-poaching patrol planning, the probability of finding a poacher's snare at some location is not known with certainty. 

Accordingly, methods have been developed to learn RMAB policies \textit{online}, assuming no \textit{a priori} knowledge \citep{jung2019thompson,wang2020restless}. However, these methods require tens of thousands of samples to converge to good policies which is prohibitive for many real-world problems, e.g., in finite-length treatment settings such as tuberculosis \citep{mate2020collapsing} with only a few dozen rounds. Instead, real-world planners must make the most of noisy data at hand, estimating dynamics from historical data or consulting experts, inducing significant \emph{uncertainty}. RMAB techniques can be used to plan with point estimates, but we show that ignoring uncertainty can lead to arbitrarily bad policies.


To address these shortcomings and push RMABs toward wider real-world applicability, we introduce \emph{Robust RMABs}, a generalization of RMABs which allows stochastic dynamics to be specified as uncertainty intervals, rather than point estimates. This new problem is very computationally demanding, adding a combinatorial layer of complexity onto an already PSPACE-hard problem \citep{papadimitriou1994complexity}. Addressing this complexity gives rise to a rich set of challenges that necessitates the design of new techniques that not only help solve the robust objective we analyze, but also are of general interest to RMAB research.

Concretely, we plan under a \emph{minimax regret} objective, using a double oracle (DO) framework \citep{mcmahan2003planning} that has seen success in problems involving a \emph{single} Markov decision process (MDP) \citep{xu2021robust}. The DO approach casts the robust planning problem as a zero-sum game between an \emph{agent} oracle and adversarial \emph{nature} oracle. However, existing techniques fail for any non-trivially sized RMABs since the state and action spaces grow combinatorially in the number of arms $N$ and resource constraint $B$, respectively. 
Specifically, given $S$-sized state spaces for each arm, the full combinatorial problem has state space of size $S^N$ and action space--and thus policy-network output--of size $\binom{N}{B}$ (for binary-action RMAB; action space is larger with multi-action). At this size, we found that directly applying \citet{xu2021robust} to solve the full combinatorial problem as a single process fails to learn good policies for RMABs as small as $N=5$ arms, with $B=3$ 
and $S=2$. Moreover, under the minimax regret objective, the nature oracle is a particularly difficult challenge as it requires jointly searching the RMAB policy space and the continuous, uncertain space of transition probabilities. Previously, this objective has been posed as a non-stationary RL problem and solved heuristically with a single policy network \citep{xu2021robust}.
We improve the nature oracle by formulating it as a multi-agent RL problem and develop a novel solution method for RMABs. In summary, our contributions are:


\begin{enumerate}[leftmargin = *]
  \setlength\itemsep{0em}
\item We introduce the Robust RMAB problem with interval uncertainty over arm dynamics and develop techniques to solve a minimax regret objective via double oracle.
\item To enable the DO approach, we introduce DDLPO, a novel deep RL algorithm for RMABs, of general interest. DDLPO tackles the combinatorial complexity of RMABs by learning an auxiliary ``$\lambda$-network'' in tandem with individual arm policy networks, which greatly reduces training sample complexity. The procedure implements the reward-maximizing agent oracle, has convergence guarantees, and solves RMABs with multiple action types \citep{killian2021multiAction,glazebrook2011general}, the first deep RL procedure to do so. DDLPO also easily extends to more general weakly-coupled MDPs \citep{adelman2008relaxations,hawkins2003langrangian} and enables computing continuous-action policies, a previously unstudied RMAB direction.
\item We formulate the non-stationary regret-maximizing nature oracle as a multi-agent RL (MARL) problem, a framework of potential general interest in robust planning. We solve this problem in the combinatorially hard RMAB setting by extending DDLPO to include a shared critic and a continuous-action policy network for nature's selection of the uncertain transition dynamics.
\end{enumerate}

\section{Related Work}
\label{sec:related_work}




\paragraph{RMABs} The reward-maximizing, binary-action RMAB problem was introduced by \citet{whittle1988restless}. His widely used Whittle index policy \citep{mate2020collapsing,glazebrook2006some,bagheri2015restless} is asymptotically optimal under \textit{indexability} \citep{weber1990index}. \citet{glazebrook2011general} and \citet{hodge2015asymptotic} extended the Whittle index to multi-action RMABs with special monotonic structure, while \citet{killian2021multiAction} gave a more general Lagrange-based method. \citet{hawkins2003langrangian} studied methods for weakly coupled Markov decision processes (WCMDP), which generalize multi-action RMABs to have multiple constraints, and propose Lagrangian solutions for small problems. \citet{adelman2008relaxations} and \citet{gocgun2012lagrangian} followed by providing better solutions to WCMDPs but sacrifice scalability. All these works assumed precise knowledge of stochastic dynamics. Some recent works have studied online RMABs with unknown dynamics but all have prohibitively large sample complexity ~\citep{gafni2020learning,jung2019regret,biswas2021learn,killian2021Q}. None consider robust planning under environment uncertainty, which we address. 

Our work also relates to learning algorithms for \textit{stochastic} multi-armed bandit (MAB) problems \citep{min2020policy, boutilier2020differentiable,kuleshov2014algorithms}. However, since stochastic MABs follow a stateless reward process, learning algorithms utilize the fact that the true optimal policy simply selects the top $B$ reward-producing arms each round. Conversely, the arms in restless MABs have reward processes that follow MDPs, so the top $B$ arms to play each round is state- and action-dependent and 
constantly evolving, making both the learning and the planning problems much more challenging, and which our algorithms address.

\paragraph{RL for RMABs} A few recent works learn Whittle indices for indexable binary-action RMABs using (i)~deep RL (DRL) \citep{nakhleh2020neurwin} and (ii)~tabular Q-learning ~\citep{biswas2021learn,fu2019towards,avrachenkov2020whittle}. \citet{killian2021Q} take tabular Q-learning to the multi-action setting. In contrast, our DRL approach provides a more general solution to binary and multi-action RMAB domains, not requiring indexability or problem structure, and is far more scalable than tabular methods. We are also the first to handle continuous-action RMABs, key to the nature oracle. Also related is the space of combinatorial RL. However, most existing algorithms consider single-shot problems, e.g., traveling salesman \citep{kool2018attention,khalil2017learning}, which lack a notion of future state that is critical to solving any version of RMAB, and none accommodate the general cost/budget structure of multi-action RMAB \citep{song2019solving}; our methods address these limitations.

\paragraph{Robust planning} Work on robust planning in RL mainly focuses on maximin reward via robust adversarial RL \citep{pinto2017robust} or multi-agent RL (MARL) \citep{lanctot2017unified,li2019robust}, but maximin reward leads to overly conservative policies \citep{nguyen2014regret}. The minimax regret criterion \citep{braziunas2007minimax} avoids this pitfall, but this objective is challenging with very large or continuous strategy spaces. This can be addressed with the DO approach proposed by \citet{mcmahan2003planning} which explores a small subset of strategies while still guaranteeing optimal convergence \citep{gilbert2017double}. Subsequently, DO has been extended to optimize MARL problems with multiple selfish agents \citep{lanctot2017unified}. Recently, \citet{xu2021robust} used DO to solve a single Markov decision process (MDP) minimax-regret planning problem and used RL to implement the oracles. However, when applied to RMABs, the number of outputs in their policy network grows exponentially, as does the size of the state space being learned, both of which require prohibitively long training times beyond trivially sized RMABs. Accordingly, we found that their RL algorithms failed to scale past $N=5$ arms and $S=2$ states, whereas we show in Sec.~\ref{sec:experiments} that our algorithms solve problems that are orders of magnitude larger. Additionally, their approach is designed only for continuous state/action spaces, whereas our approach can find robust policies for any combination of discrete \emph{or} continuous state/action spaces. We accomplish this via our novel formulation of the nature oracle as a MARL problem, which decomposes the causes of non-stationarity, i.e., agent and nature, and learn them with separate networks.




\begin{figure*}[t]
\centering
\includegraphics[width=0.95\linewidth]{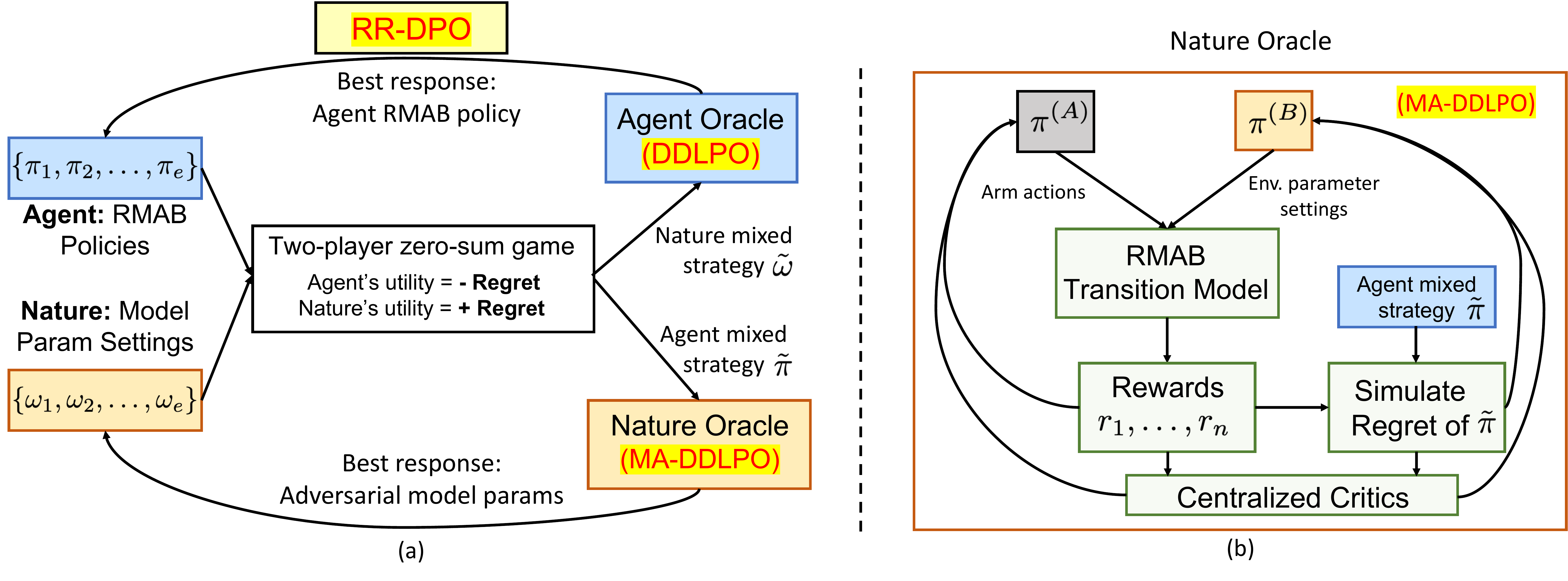} 
\caption{\textbf{(a)}~Proposed framework for solving the Robust RMAB problem. The main loop follows a DO approach to iteratively compute a minimax regret optimal RMAB policy where each oracle is a novel DRL algorithm for RMABs. 
\textbf{(b)}~The nature oracle: a novel multi-agent RL formulation of RMAB, that tackles non-stationarity with a centralized critic.}
\label{fig:concept} 
\end{figure*}

\section{Preliminaries}
\label{sec:preliminaries}

We consider the multi-action RMAB setting with $N$ arms \citep{killian2021multiAction,glazebrook2011general}, which generalizes classical binary-action RMABs \citep{whittle1988restless}.\footnote{Our approaches also easily extend to weakly-coupled MDPs, which allow multiple budget constraints \citep{hawkins2003langrangian}, as well as to continuous-action RMABs, previously unstudied.} Each arm $n\in [N]$ follows an MDP $(\mathcal{S}_n, \mathcal{A}_n, \mathcal{C}_n, T_n, R_n, \beta)$, where $\mathcal{S}_n$ is a set of finite, discrete states; $\mathcal{A}_n$ is a set of finite, discrete actions; $\mathcal{C}_n : \mathcal{A}_n \xrightarrow[]{} \mathbb{R}$ defines action costs, where $\mathcal{C}_n[0] = 0$ encodes a no-cost ``passive action'' for all arms; $T_n: \mathcal{S}_n\times \mathcal{A}_n \times \mathcal{S}_n \xrightarrow[]{} [0,1]$ gives the probability of transitioning from one state to another given an action; $R_n:\mathcal{S}_n \xrightarrow[]{} \mathbb{R}$ is a reward function; and $\beta \in [0, 1)$ is the discount factor. For ease of exposition, let $\mathcal{S}_n, \mathcal{A}_n, \mathcal{C}_n,$ and $R_n$ be the same for all $n\in[N]$, and thus drop the subscript $n$, though all methods apply to the general case. Let $\bm{s}$ be an $N$-length vector of states over all arms and let $\bm{A} \in \{0,1\}^{N\times |\mathcal{A}|}$ be a decision matrix that one-hot-encodes the action taken on each arm. The planner computes policies $\pi$ which map states $\bm{s}$ to actions $\bm{A}$ with the constraint that the sum cost of actions is less than a budget $B$ in every round $t \in [H]$. 


We extend multi-action RMABs to the robust setting in which the exact transition probabilities are unknown. Instead, the transition dynamics $T_n$ of each arm $n \in[N]$ are determined by a set of parameters $\omega_n \in \Omega_n$, each within a given interval uncertainty $\underline{\overline{\omega}}_n:=[\underline{\omega}_{n}, \overline{\omega}_{n}]$. Let $\omega$ be a given parameter setting such that $\omega_{n} \in \underline{\overline{\omega}}_n$ for all $n \in [N]$. Let $G(\pi,\omega) = \mathbb{E}[\sum_{t=1}^{H}\beta^t \sum_{n\in [N]}R(\bm{s}^n_t) \mid \pi, \omega]$ be the planner's expected discounted reward under $\pi$ and $\omega$, where $\bm{s}_n^t$ is the state of arm $n$ at time $t$. Then, \emph{regret} is defined: 
\begin{align}
L(\pi,\omega) = G(\pi^\star_{\omega},\omega) - G(\pi,\omega) \ ,
\label{eq:regret}
\end{align}
where $\pi^\star_{\omega}$ is the optimal reward-maximizing policy under $\omega$. In our robust setting, our objective is to compute a policy~$\pi^{\dagger}$ that minimizes the maximum regret~$L$ possible for any realization of $\omega$, i.e.:
\begin{align}
    \pi^{\dagger} = \min_{\pi}\max_{{\omega}}{L(\pi,{\omega})} \ .
    \label{eq:minimax}
\end{align}
This problem is computationally expensive to solve since simply computing a policy~$\pi$ that maximizes the reward $G(\pi,\omega)$ is PSPACE-hard \citep{papadimitriou1994complexity} even when the $T_n$ are known, i.e., $\omega$ is given. 

A more tractable approach for computing multi-action RMAB policies $\pi$ is to utilize the Lagrangian relaxation \citep{hawkins2003langrangian,killian2021multiAction}, reproduced below.
For a given $\omega$, the optimal policy $\pi^\star_{\omega}$ maximizes the constrained Bellman equation:
\begin{align}\label{eq:combined_value_function}
    J(\bm{s}) &= \max_{\bm{A}^c}\left\{\sum_{n=1}^{N} R(\bm{s}_n) + \beta \mathop{\mathbb{E}}_{\omega}[J(\bm{s}^\prime) \mid \bm{s}, \bm{A}^c]\right\} \\
    & \hspace{-6mm}\text{where } \bm{A}^c \subseteq \bm{A} \nonumber \\ 
    \text{s.t. } &\sum_{n=1}^{N}\sum_{j=1}^{|\mathcal{A}|} \bm{A}_{nj}c_{j} \le B
    \qquad
    \sum_{j=1}^{|\mathcal{A}|} \bm{A}_{nj} = 1 \hspace{2mm} \forall n \in [N]  \nonumber
\end{align}
where $\bm{A}_{nj} = 1$ if the $j^\text{th}$ action is taken on arm $n$ (else 0) and $c_{j} \in \mathcal{C}$ is the $j^\text{th}$ action cost. We then take the Lagrangian relaxation of the budget constraint~\citep{hawkins2003langrangian}, giving:
\begin{align}
    &J(\bm{s}, \lambda^\star) = \min_{\lambda} \left( \frac{\lambda B}{1-\beta} + \sum_{n=1}^{N}\max_{j\in|\mathcal{A}|}\{Q_n(\bm{s}_n, a_{nj}, \lambda)\} \right) \label{eq:decoupled_value_func} \\
    &\quad \text{where }\hspace{1mm} Q_n(\bm{s}_n, a_{nj}, \lambda) =
    R(\bm{s}_n) - \lambda c_{j} + \nonumber \\ 
    &\quad \qquad \beta \mathbb{E}_{\omega} \left[ Q_n(\bm{s}_n^{\prime}, a_{nj}, \lambda) \mid \pi^{La}_{\omega}(\lambda) \right] \ . \label{eq:arm_value_func_lagrange}
\end{align}
Here, $a_{nj}$ is the $j^\text{th}$ action of arm $n$, $Q$ is the state-action value function, and  $\pi^{La}_{\omega}(\lambda)$ is the optimal policy for a given $\lambda$. The key insight is that this relaxation decouples the value functions of the arms, except for the shared $\lambda$, i.e., for a given value of $\lambda$, all $Q_n$ could be solved via $N$ individual value iterations. However, finding and setting $\lambda:= \lambda^\star$ is critical to finding good policies for multi-action RMABs \citep{killian2021multiAction,glazebrook2011general}, where $\pi^{La}_{\omega}(\lambda^\star)$ is used to recover a policy that respects the original budget constraint by solving a knapsack with $Q_n(\bm{s}_n, a_{nj}, \lambda^\star)$ as values, $\mathcal{C}$ as weights, and the constraints of Eq.~\ref{eq:combined_value_function}, then taking the actions according to the $Q_n$ in the solved knapsack. The knapsack solution finds the combination of actions with the largest sum of learned $Q_n(\bm{s}_n, a_{nj}, \lambda^\star)$ values which still respects the budget. The integer program for the knapsack is given in Appendix~\ref{sec:appendix:knapsack} and has time complexity $\mathcal{O}(N|\mathcal{A}|B)$ \citep{killian2021multiAction}.


\section{Solving Robust RMAB\lowercase{s}}
\label{sec:robust-rmab}

We now build our approach for finding robust RMAB policies, visualized in Fig.~\ref{fig:concept}(a). We use an iterative DO approach which achieves the minimax regret objective of Eq.~\ref{eq:minimax} by casting the optimization problem as a zero-sum game between two players: an \textit{agent} which learns policies $\pi$ to minimize regret, and an adversarial \textit{nature} which selects environment parameters $\omega$ to maximize regret of the agent. 
In this two-player game, the \textit{pure strategy} space for the agent is the set of all feasible RMAB policies $\pi$ that respect the budget constraint. The pure strategy space for nature is a continuous, closed set of parameters $\omega$ within the given uncertainty intervals. The algorithm maintains a pure strategy set for the agent and nature (Fig.~\ref{fig:concept}(a) left boxes); each iteration, these strategy sets are used to compute a \textit{mixed strategy}---i.e., a probability distribution over pure strategies---Nash equilibrium in a regret game (Fig.~\ref{fig:concept}(a) center). Each oracle then learns a best response against the opponent's mixed strategy to add to its strategy set (Fig.~\ref{fig:concept}(a) right boxes). 

The agent oracle's goal is to find an RMAB policy $\pi$, or pure strategy, to minimize regret (Eq.~\ref{eq:regret}) given a nature mixed strategy $\tilde{\omega}$. That is, the agent minimizes $L(\pi,\tilde{\omega})$ w.r.t.\ $\pi$, while $\tilde{\omega}$ is constant. Recall from Eq.~\ref{eq:regret} that
$L(\pi,\tilde{\omega}) = G(\pi^\star_{\tilde{\omega}},\tilde{\omega}) - G(\pi,\tilde{\omega})$.
Since $\tilde{\omega}$ and $\pi^\star_{\tilde{\omega}}$ are constant, then the first term $G(\pi^\star_{\tilde{\omega}},\tilde{\omega})$ is also constant. Thus minimizing $L(\pi,\tilde{\omega})$ is equivalent to maximizing the second term $G(\pi,\tilde{\omega})$, which is maximal at $\pi=\pi^\star_{\tilde{\omega}}$. In other words, the agent oracle must compute an optimal reward-maximizing policy w.r.t.\ $\tilde{\omega}$. Such a reward-maximizing objective aligns with existing RL techniques, but still requires that we address the challenge of learning in the combinatorial state and action spaces of the RMAB. To address this challenge, \emph{we propose a new RL method which decomposes the RMAB into $N$ per-arm learning problems and a complementary $\lambda$-network learning problem}, which together learn to spend limited budget where it will give the best return, detailed in Sec.~\ref{sec:rl-rmab}.

Conversely, the nature oracle seeks to find a parameter setting $\omega$, or pure strategy, that maximizes the agent's regret given a mixed strategy $\tilde{\pi}$, i.e., maximize $L(\tilde{\pi},\omega)$ with respect to $\omega$, while $\tilde{\pi}$ is fixed. This objective is even more challenging because both $G(\pi^\star_{\omega},\omega)$ and $G(\tilde{\pi},\omega)$ are functions of $\omega$. Most critically, computing $G(\pi^\star_{\omega},\omega)$ requires obtaining an optimal policy $\pi^\star_{\omega}$ as $\omega$ changes in the optimization---this amounts to a planning problem in which an agent must learn an optimal policy while the environment changes, controlled by $\omega$, making the nature oracle difficult to solve. Moreover, in the interval uncertainty setting we consider, $\omega$ is defined by a space of continuous values; thus nature's pure strategy space is infinite, making the problem even more complex, since it cannot be exhaustively searched. 

\emph{To tackle this complexity we propose a novel method for implementing the regret-maximizing nature oracle by casting it as an MARL problem}. The approach, visualized in Fig.~\ref{fig:concept}(b), trains one auxiliary agent to solve for a policy $\pi^\star_{\omega}$ ($\pi^A$ in Fig.~\ref{fig:concept}(b)), needed to compute $G(\pi^\star_{\omega},\omega)$ in the regret term, and simultaneously trains a second agent to learn worst-case parameters $\omega$ ($\pi^B$ in Fig.~\ref{fig:concept}(b)) that minimize $G(\tilde{\pi},\omega)$---together, these will maximize the regret  $L(\tilde{\pi},\omega)$. With this MARL setup, we mitigate nonstationarity through centralized critic networks which allow each agent to include the other's actions in their learned state space.
Solving a MARL problem requires an RL algorithm to optimize the underlying policy, so we first introduce our novel RL approach, DDLPO, to solve RMABs (Sec.~\ref{sec:rl-rmab}) as a part of our agent oracle and then use the algorithm as the backbone of our nature oracle (Sec.~\ref{sec:marl-rmab}).


\subsection{Agent Oracle: Deep RL for RMAB}
\label{sec:rl-rmab}

\begin{algorithm}[t]
\caption{DDLPO}
\label{alg:ddlpo}
\begin{flushleft}
\textbf{Input}: Initial state $\bm{s}_0$, nature mixed strategy $\tilde{\omega}$, \\
\texttt{n\_epochs}, \texttt{n\_subepochs}, \texttt{n\_steps}
\end{flushleft}
\begin{algorithmic}[1] 
\STATE Init.~policy networks $\theta_{n}$ for each arm $n \in [N]$
\STATE Init.~critic networks $\phi_{n}$ for each arm $n \in [N]$
\STATE Init.~$\lambda$-network $\Lambda$
\STATE Init.~\texttt{buff} = [] and $\bm{s}=\bm{s}_0$
\FOR{$\textit{epoch} = 1, 2, \ldots, \texttt{n\_epochs}$}
\STATE Sample $\lambda = \Lambda(\bm{s})$
\STATE Sample $\omega \sim \tilde{\omega}$
\FOR {$\textit{subepoch} = 1, \ldots, \texttt{n\_subepochs}$}
\FOR {timestep $t = 1, \ldots, \texttt{n\_steps}$}
\STATE Sample actions $a_n \sim \theta_n(s_n, \lambda)$ \hspace{1mm} $\forall n \in [N]$
\STATE $\bm{s}^\prime, \bm{r} = \texttt{Simulate}(\bm{s}, \bm{a}, \omega)$
\STATE Add tuple $(\bm{s}, \bm{a}, \bm{r}, \bm{s}^\prime, \lambda)$ to \texttt{buff} 
\STATE $\bm{s} = \bm{s}^\prime$
\ENDFOR
\STATE Update each $(\theta_n$, $\phi_n)$ pair via PPO, using trajectories in \texttt{buff}
\ENDFOR
\STATE Update $\Lambda$ via Prop.~\ref{thm:lambda_update} with costs of final subepoch
\ENDFOR
\STATE \textbf{return} $\theta_1, \ldots, \theta_N$, $\phi_1, \ldots, \phi_N$ and $\Lambda$
\end{algorithmic}
\end{algorithm}

\begin{algorithm}[t]
\caption{DDLPO-Act}
\label{alg:ddlpo-act}
\begin{flushleft}
\textbf{Input}: State $\bm{s}$, costs $\mathcal{C}$, budget $B$, agent actor, critic, and $\lambda$ networks $\theta_1,\ldots,\theta_N$, $\phi_1,\ldots,\phi_N$, $\Lambda$, selection method $\alpha$
\end{flushleft}
\begin{algorithmic}[1] 
\STATE $\lambda$ = $\Lambda(\bm{s})$
\IF{$\alpha$ == `GreedyProba'}
    \STATE $p_n = \theta_n(s_n,\lambda)$ \hspace{1mm} $\forall n \in [N]$ \COMMENT{Action distr.~of arm $n$}
    \color{black}\STATE $\bm{a}$ = \texttt{GreedyProba}$(\bm{p}, \mathcal{C}, B)$ \COMMENT{Greedily select highest probability actions until budget $B$ is reached}
\ELSIF{$\alpha$ == `QKnapsack'}
    \STATE $q_{nj} = \phi_n(s_n,a_{nj},\lambda)$ \hspace{1mm} $\forall n \in [N], \forall j \in [|\mathcal{A}|]$
    \STATE $\bm{a}$ = \texttt{QKnapsack}$(\bm{q}, \mathcal{C}, B)$ \COMMENT{Solve knapsack in Appendix \ref{sec:appendix:knapsack}}
\ELSIF[Binary action only]{$\alpha$ == `Whittle'}\color{black}
    \STATE $\bm{a} = \textsc{BinaSearch}(\bm{s}, B, \phi_1,...,\phi_N)$ \COMMENT{Appendix \ref{sec:appendix:knapsack}}
\ENDIF
\color{black}\STATE \textbf{return} $\bm{a}$
\end{algorithmic}
\end{algorithm}


Existing DRL approaches can be applied to the objective in Eq.~\ref{eq:combined_value_function}, but, as detailed in Sec.~\ref{sec:related_work}, they fail to scale past trivially sized RMAB problems since the action and state spaces grow exponentially in $N$.
To overcome this, we develop a novel DRL algorithm that instead solves the decoupled problem (Eq.~\ref{eq:decoupled_value_func}). The key benefit of decoupling is to render policies and $Q$ values of each arm independent, allowing us to learn $N$ independent networks with \textit{linearly sized state and action spaces, relieving the combinatorial burden of the learning problem}.
However, this decoupling approach introduces a new technical challenge in solving the dual objective which maximizes over policies but minimizes over $\lambda$, as discussed in Sec.~\ref{sec:preliminaries}.

To solve this, we derive a dual gradient update procedure that iteratively optimizes each objective as follows: (1)~holding $\lambda$ constant, learn $N$ independent policy networks via policy gradient, augmenting the state space to include $\lambda$ as input, as in Eq.~\ref{eq:decoupled_value_func}; (2)~use sampled trajectories from those learned policies as an estimate to update $\lambda$ towards its minimizing value via a novel gradient update rule. Another challenge is that $\lambda^\star$ of Eq.~\ref{eq:decoupled_value_func} depends on the current state of each arm---therefore, a key element of our approach is to learn this function $\lambda^\star(\bm{s})$ concurrently with our iterative optimization, using a neural network we call the $\lambda$-network that is parameterized by $\Lambda$. To train the $\lambda$-network, we use the following gradient update rule.
\begin{restatable}[]{proposition}{lambdaUpdate}
\label{thm:lambda_update}
To learn the value $\lambda$ that minimizes Eq.~\ref{eq:decoupled_value_func} given a state $\bm{s}$, the $\lambda$-network, parameterized by $\Lambda$, should be updated with the following gradient rule:
\begin{equation}
\begin{aligned}
    \Lambda_t = \Lambda_{t-1} - \alpha \left( \frac{B}{1-\beta} + \sum_{n=1}^{N}D_n(s_n, \lambda_{t-1}(\bm{s})) \right) 
\end{aligned}
\end{equation}
where $\alpha$ is the learning rate and $D_n(s_n, \lambda)$ is the negative of the expected $\beta$-discounted sum of action costs for arm $n$ starting at state $s_n$ under the optimal policy for arm $n$ for a given value of $\lambda$.
\end{restatable}

As $D_n$ lacks a closed form, the key insight we make is that it can be estimated by sampling multiple rollouts of the policy networks of all arms during training. As long as arm policies are trained for adequate time on the given value of $\lambda$, the gradient estimate will be accurate, i.e., $D_n(s_n, \lambda_{t-1}(\bm{s})) \approx -\sum_{k=0}^{K-1} \beta^k c^{k}_{n}$ where $K$ is the number of samples collected in an epoch and $c^{k}_{n}$ is the action cost of arm~$n$ in round~$k$. Moreover, this procedure will converge to the optimal parameters~$\Lambda^\star$ if the arm policies are optimal.

\begin{restatable}[]{proposition}{lambdaConvergence}
\label{thm:lambda_convergence}
Given arm policies corresponding to optimal $Q$-functions, 
Prop.~\ref{thm:lambda_update} will lead $\Lambda$ to converge to the optimal as the number of training epochs and $K\xrightarrow[]{}\infty$.
\end{restatable}

Proofs are given in Appendix \ref{sec:appendix:proofs}. One interesting feature of this update rule is that to collect samples that reflect the proper gradient, the RMAB budget must not be imposed \textit{at training time}---rather, the policy networks and $\lambda$-network must be allowed to learn to play the Lagrange policy of Eq.~\ref{eq:decoupled_value_func}, which learns to spend the correct budget in expectation, via our iterative update procedure. Therefore, at training time, we sample actions randomly according to the actor network distributions, without imposing the budget constraint. However, \textit{at test time, we always take actions in a way that respects the budget constraint} as described in Alg.~\ref{alg:ddlpo-act}. Alg.~\ref{alg:ddlpo-act} chooses actions either by (1) selecting greedily by the probabilities of the arm actor networks (2) using the learned $Q(\lambda)$-functions of the arm critic networks to follow the Q-value-maximizing knapsack procedure (Appendix \ref{sec:appendix:knapsack}), or (3) in binary-action settings, using the $Q(\lambda)$-functions to follow a binary search procedure such that selected actions are equivalent to the Whittle index policy (Appendix \ref{sec:appendix:knapsack}).

In theory, the policy networks could be trained via any DRL procedure that ensures the above characteristics for training the $\lambda$-network. In practice, we train with proximal policy optimization (PPO) \citep{schulman2017proximal}, a state-of-the-art policy gradient approach. Importantly, PPO is also flexible enough to handle both discrete and continuous actions which is necessary for the nature oracle. 

Finally, to enable our iterative, dual-update procedure in practice, we need a mechanism to both (1)~explore new arm policy actions after an update to $\Lambda$, then (2)~exploit learned policy actions to develop good gradient estimates for $\Lambda$. We navigate this important trade-off by adding an entropy regularization term to the policy networks losses, controlled via a cyclical temperature parameter. 
We call our algorithm Deep Distributed Lagrange Policy Optimization (DDLPO), provide pseudocode in Algorithm~\ref{alg:ddlpo}, and include more implementation details in Appendix \ref{sec:appendix:hyperparams-and-details}.


\subsection{Nature Oracle: Multi-Agent RL}
\label{sec:marl-rmab}

Armed with a DRL procedure for learning RMAB policies, we now develop the MARL procedure, which we call MA-DDLPO, to implement the nature oracle. Recall that the challenge of the nature oracle is to jointly optimize a policy $\pi^\star_{\omega}$ and environment parameters~$\omega$. We propose to solve this optimization using MARL, designed to handle this form of non-stationarity \citep{lowe2017multi} via centralized critics. In our MARL setup, each of two ``players'' (i.e., the ``multiple agents'') will aim to compute $\pi^\star_{\omega}$ and $\omega$, respectively, with separate objectives. The procedure is visualized in Fig.~\ref{fig:concept}(b).

To implement the MARL nature oracle, we introduce two new players $A$ and $B$. Player $A$ is an \textit{auxiliary player} whose goal is to optimize the RMAB policy $\pi^\star_{\omega}$ given a changing $\omega$, i.e., the first term of regret (Eq.~\ref{eq:regret}. We call $A$ auxiliary because its learned policy will never be used outside the nature oracle; $A$ is only used to assist the nature oracle in computing the regret associated with a given $\omega$. Alternatively player $B$ is an adversarial player whose goal is the same as that of the nature oracle itself, i.e., to find parameters $\omega$ that maximize regret of the current agent mixed strategy $\tilde{\pi}$. We define a shared transition function for the environment in which the players act $T: \mathcal{S} \times \mathcal{A}_A \times \mathcal{A}_B \xrightarrow[]{} \mathcal{S}$. Here, $\mathcal{A}_A$ is the action space of the underlying multi-action RMAB. At a given state~$\bm{s}$, the action space $\mathcal{A}_B$ defines for player~$B$ actions~$\omega$ which, in general, depend on $\bm{s}$. That is, at each step, player $B$ selects environment parameters $\omega$, and thus transition probabilities that will influence the outcome of player $A$'s actions. We adopt the centralized critic idea from multi-agent PPO \citep{yu2021surprising} to our RMAB setting to create MA-DDLPO. 
A notable strength of our MARL approach is that it allows the discrete-space policy of player $A$ and the continuous-space policy of player $B$ to be learned by separate networks, simplifying training compared to an alternative combined-network approach. Moreover, our choice to use PPO offers a convenient way to learn both types of policies as separate networks, while utilizing a single framework of update rules. 

\begin{algorithm}[tb]
\caption{MA-DDLPO}
\label{alg:maddlpo}
\textbf{Input}: Agent mixed strategy $\tilde{\pi}$, \texttt{n\_epochs}, \\ 
\texttt{n\_subepochs}, \texttt{n\_steps}, \texttt{n\_sims}
\begin{flushleft}
\algsetup{indent=0.75em}
\begin{algorithmic}[1] 
\STATE Init.~player A: arm policy networks $\theta_{n}^{(A)}$ and arm critic networks $\phi_{n}^{(A)}~\forall n \in [N]$, and $\lambda$-network $\Lambda$
\STATE Init.~player B: environment parameter policy network $\theta^{(B)}$, critic network $\phi^{(B)}$
\STATE Init.~\texttt{buff} = []
\FOR{$\textit{epoch} = 1, 2, \ldots, \texttt{n\_epochs}$}
\STATE Sample $\bm{s}$ at random \\
\STATE Sample $\lambda = \Lambda(\bm{s})$ \\
\FOR {$\textit{subepoch} = 1, \ldots, \texttt{n\_subepochs}$}
\FOR{$t = 1, \ldots, \texttt{n\_steps}$}
\STATE Sample $a_n^{(A)} \sim \theta_{n}^{(A)}(s_n, \lambda)$ for each $n \in [N]$
\STATE Sample $\omega^{(B)} \sim \theta^{(B)}(\bm{s})$ 
\STATE $\bm{r}^{(A)}, \bm{s}^\prime = \textsc{simulate}(\bm{s}, \bm{a}^{(A)}, \omega^{(B)})$
\STATE $\tilde{r} = \textsc{simulate}(\bm{s}, \tilde{\pi}(\bm{s}), \omega^{(B)}, \texttt{n\_sims})$ \hspace{15mm} \COMMENT{(mean of \texttt{n\_sims} 1-step rollouts of $\tilde{\pi}$)}
\color{black}\STATE $r^{(B)} = \left(\sum_{n \in [N]}r_n^{(A)}\right) - \tilde{r}$ \COMMENT{(regret of $\tilde{\pi}$)}
\color{black}\STATE Add $(\bm{s}, \bm{a}^{(A)}, \omega^{(B)}, \bm{r}^{(A)}, r^{(B)}, \bm{s}^\prime, \lambda)$ to \texttt{buff}
\STATE $\bm{s} = \bm{s}^\prime$
\ENDFOR
\STATE Update each $(\theta_n^{(A)}$, $\phi_n^{(A)})$ pair using trajectories in \texttt{buff}. $\phi_n^{(A)}$ get $\omega^{(B)}$ as part of state
\ENDFOR
\STATE Update $\Lambda$ via Prop.~\ref{thm:lambda_update} with costs of final subepoch
\STATE Update $\theta^{(B)}, \phi^{(B)}$ using trajectories in \texttt{buff}. $\phi^{(B)}$ gets $\bm{a}^{(A)}$ as part of state
\ENDFOR
\STATE \textbf{return} $\theta^{(B)}$
\end{algorithmic}
\end{flushleft}
\end{algorithm}

A critical step is then to define the rewards for players $A$ and $B$ to match their objectives. Since player $A$'s objective is to find $\pi^\star_{\omega}$, it adopts the reward defined by the underlying RMAB, i.e., ${R}^{(A)}(\bm{s}) = \sum_{n=1}^N R_n(\bm{s})$. However, player $B$'s objective is to learn the regret-maximizing parameters $\omega$. This objective is challenging because it requires 
computing and optimizing over the returns of the fixed input policy $\tilde{\pi}$ with respect to all possible $\omega$, which is in general non-convex. In practice, to estimate the returns of $\tilde{\pi}_\omega$, we execute a series of roll-outs against player $B$'s current action. 
That is, given $\bm{s}$ at a given round, we sample an action from $\tilde{\pi}_{\omega}$ and the next state $\bm{s^\prime}$, and define the \textit{regret-based} reward of player~$B$, as ${R}^{(B)} = \sum_{n=1}^N R_n(s_n) - \frac{1}{Y}\sum_{y=1}^{Y}r_y^{\tilde{\pi},\omega}$, where $r_y^{\tilde{\pi},\omega}$ is the reward from each of $Y$ one-step Monte Carlo simulations of the mixed strategy $\tilde{\pi}$ in $\omega$.

To train the policies, player $A$ has the same policy network architecture as DDLPO, i.e., $N$ discrete policy networks and one $\lambda$-network, and the player $B$ actor network is a single continuous-action policy network. Since players $A$ and $B$ have separate reward functions, they have their own critic networks, but these critics are \emph{centralized} in that they both take the actions of the other as input. Other than the centralized critic, player $A$ is trained the same way as DDLPO, and player $B$ is trained in a standard PPO fashion. In practice, to ensure good gradient estimates for player $A$'s $\lambda$-network in MA-DDLPO, we keep player $B$'s network---and thus the environment---constant between $\Lambda$ updates, updating $B$'s network with the same frequency as the $\lambda$-network updates. Pseudocode for MA-DDLPO is given in Alg.~\ref{alg:maddlpo} and further details of its implementation are given in Appendix \ref{sec:appendix:hyperparams-and-details}. 

\begin{algorithm}[t]
\caption{RR-DPO}
\label{alg:full-alg}
\textbf{Input}: Environment simulator and parameter uncertainty intervals $\overline{\underline{\omega}}_n$ for all $n \in [N]$ \\
\textbf{Parameters}: Convergence threshold $\varepsilon$
\\
\textbf{Output}: Agent mixed strategy $\tilde{\pi}$
\begin{algorithmic}[1] 
\STATE $\Omega_0 = \{\omega_0\}$, with $\omega_0$ selected at random
\STATE $\Pi_0 = \{\pi_{B_1}, \pi_{B_2}, \ldots\}$, where $\pi_{B_i}$ are baseline and heuristic strategies \\
\FOR{epoch $e = 1, 2, \ldots$}
\STATE Solve for $(\tilde{\pi}_e, \tilde{\omega}_e)$, mixed Nash equilibrium of regret game with strategy sets $\Omega_{e-1}$ and $\Pi_{e-1}$ \\
\STATE $\pi_e = \textsc{DDLPO}(\tilde{\omega}_e)$ \\
\STATE $\omega_e = \textsc{MA-DDLPO}(\tilde{\pi}_e)$ \\
\STATE $\Omega_e = \Omega_{e-1} \cup \{\omega_e\}, \Pi_e = \Pi_{e-1} \cup \{\pi_e\}$
\IF{$L(\tilde{\pi}_e, \omega_e) - L(\tilde{\pi}_{e-1}, \tilde{\omega}_{e-1}) \leq \varepsilon$ and $L(\pi_e, \tilde{\omega}_e) - L(\tilde{\pi}_{e-1}, \tilde{\omega}_{e-1}) \leq \varepsilon$}
\STATE \textbf{break}
\ENDIF
\ENDFOR
\STATE \textbf{return} $\tilde{\pi}_e$
\end{algorithmic}
\end{algorithm}

\subsection{Minimax Regret RMAB Double Oracle}
We now have all the pieces needed to run our robust algorithm, Robust RMABs via Deep Policy Oracles (RR-DPO), visualized in Fig.~\ref{fig:concept}(a), with pseudocode presented in Algorithm~\ref{alg:full-alg}, adapted from the MIRROR framework \citep{xu2021robust}. 
We use DDLPO to instantiate the agent oracle, MA-DDLPO for the nature oracle, and run RR-DPO until the improvement in value for each player is within a tolerance~$\varepsilon$ or until a set number of iterations.

We now establish conditions under which RR-DPO converges to the minimax regret--optimal policy in finite iterations. 
In the binary-action setting, assuming each oracle returns true best responses, and under an analytical condition that is straightforward to achieve, i.e., finite pure strategy sets:\footnote{Straightforward to achieve for nature oracle via discretization.}

\begin{restatable}[]{proposition}{rrdpoConvergenceProposition}
\label{thm:rrdpo_convergence}
RR-DPO converges in a finite number of steps to the minimax regret-optimal policy.
\end{restatable}

\noindent In addition, we empirically verify that good policies are found outside of these conditions, and that RR-DPO converges using our continuous-strategy-space nature oracle. Further, we show that a policy that maximizes reward assuming a fixed parameter set can incur arbitrarily large regret when the parameters are changed (proofs in Appendix \ref{sec:appendix:proofs}). 

\begin{restatable}[]{proposition}{regretProposition}
\label{thm:regret}
In the Robust RMAB problem with interval uncertainty, the max regret of a reward-maximizing policy can be arbitrarily large compared to a minimax regret-optimal policy.
\end{restatable}

\section{Experimental Evaluation}
\label{sec:experiments}
We first experimentally demonstrate the importance of robust planning in the presence of uncertainty using a hand-crafted synthetic domain (inspired by Prop.~\ref{thm:regret}). We then evaluate our algorithm on two challenging real-world-inspired public health planning scenarios which demonstrate the capability of our robust RMAB framework. All experiments use selection method $\alpha=$`GreedyProba' for DDLPO-Act (Alg.~\ref{alg:ddlpo-act}), which we found had the best performance.

We compare RR-DPO against five baselines. These baselines include three variations of the reward-maximizing approach from \citet{hawkins2003langrangian}, which, given fixed environment parameters $\omega$, at each step computes a Lagrange policy, then chooses actions following the knapsack procedure described in Sec.~\ref{sec:preliminaries}. The three variations are pessimistic (\textbf{HP}), mean (\textbf{HM}), and optimistic (\textbf{HO}), which assume the environment parameters are set at the lower bound, mean, and upper bound of the given intervals for each arm. We also implement \textbf{RLvMid}, which \textit{learns} (rather than computes) a policy via DDLPO assuming \textit{mean} parameters, and \textbf{Rand}, which acts randomly to fill the budget. All results are averaged over 50 random seeds and were executed on a cluster running CentOS with Intel(R) Xeon(R) CPU E5-2683 v4 @ 2.1 GHz with 8GB of RAM using Python 3.7.10. Our DDLPO implementation builds on OpenAI Spinning Up \citep{SpinningUp2018} and RR-DPO builds on the MIRROR implementation \citep{xu2021robust}, computing Nash equilibria using Nashpy 0.0.21 \citep{Knight2018}. Code is available at \url{https://github.com/killian-34/RobustRMAB} and hyperparameter settings are in Appendix \ref{sec:appendix:hyperparams-and-details}.

\subsection{Experimental Domains}

\textbf{Synthetic} demonstrates that reward-maximizing policies (RLvMid, HP, HM, HO)
may incur large regret in the presence of uncertainty. There are three binary-action arm types $\{U,V,W\}$, each with $\mathcal{C} = \{0, 1\}$, $\mathcal{S}=\{0,1\}$, $R(s)=s$, and the following transition matrix, with rows and columns corresponding to actions and next states, respectively:
\[T^n_{s=0}=
\begin{bmatrix}
    0.5  &  0.5 \\
    0.5  &  0.5
\end{bmatrix}, \hspace{2mm}
T^n_{s=1}=
\begin{bmatrix}
    1.0  &  0.0 \\
    1-p_n  &  p_n
\end{bmatrix}
\]
\[
p_U \in [0.00, 1.00],\hspace{0.5mm}
    p_V \in [0.05, 0.90],\hspace{0.5mm}
    p_W \in [0.10, 0.95]
\]
When an arm is at $s=0$, each action has equal impact on the state transition. When the arms are at $s=1$, selecting arms with high $p_n$ is optimal. This implies that policies can be specified by the order in which arms would be acted on, when they are in state $s=1$. Accordingly, $\pi_\textit{HP} = [W,V,U]$, $\pi_\textit{HM} = [W,U,V]$, and $\pi_\textit{HO} = [U,W,V]$. However, observe that there exist values of $p_n$ that can make each of the reward-maximizing policies incur large regret, e.g., for $\pi_\textit{HO}$ $p_U=0.0, p_V=0.9, p_W=0.1$ would induce an optimal policy $[V,W,U]$, which is the reverse of $\pi_\textit{HO}$. 

\textbf{ARMMAN} is a real-world \emph{maternal healthcare intervention problem} modeled as a binary-action RMAB \citep{biswas2021learn}. The goal is to select a subset of mothers each week to intervene on to encourage engagement with automated maternal health messaging. The behavior of enrolled women is modeled by an MDP with three states: Self-motivated, Persuadable, and Lost Cause. We use the summary statistics given in their paper and assume uncertainty intervals of $0.5$ centered around the transition parameters, resulting in 6 uncertain parameters per arm (details in Appendix \ref{sec:appendix:domains:armman}). Similar to the setup by \citet{biswas2021learn}, we assume 1:1:3 split of arms with high, medium, and low probability of increasing their engagement upon intervention. In our experiments, we scale the value of $N$ in multiples of $5$ to keep the same split of arm categories of 1:1:3.

\begin{figure*}[t]
    \centering
    \includegraphics[width=\textwidth]{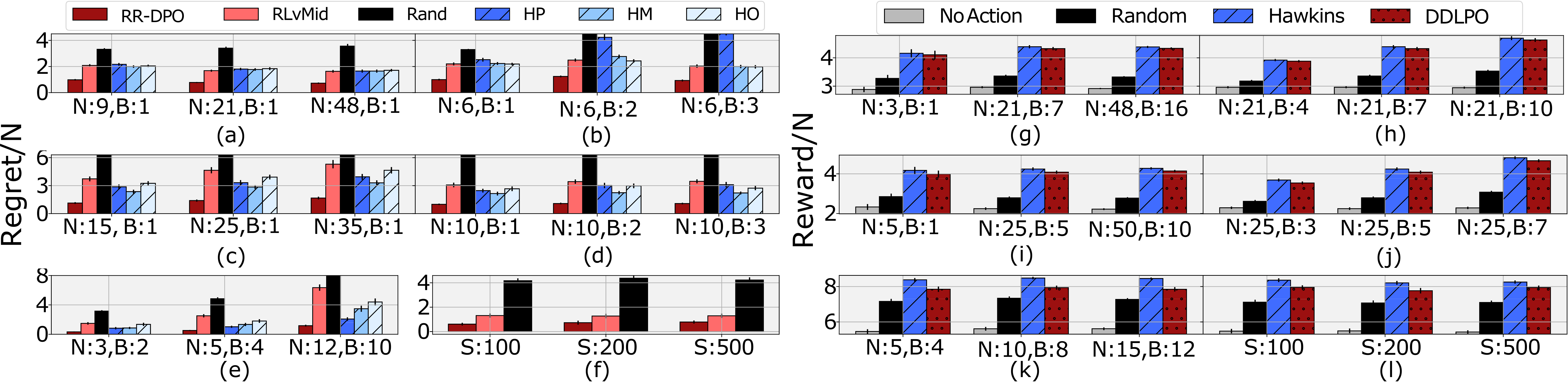}
    \caption{\textbf{(a--f)} Maximum policy regret of RR-DPO in robust setting for Synthetic (a,b), ARMMAN (c,d) and SIS (e,f) domains. Lower is better. Synthetic is scaled by 3 and ARMMAN by 5 to maintain the distributions of arm types specified in Sec.~\ref{sec:experiments}. (e)~uses $S=50$ and (f)~uses $N=5,B=4$. RR-DPO beats all baselines by a large margin across various settings. \textbf{(g--l)}~Returns of DDLPO for reward-maximizing setting (agent oracle) for synthetic (g,h), ARMMAN (i,j), and SIS (k,l) domains. Higher is better. (k)~uses $S=50$ and (l)~uses $N=5,B=4$. DDLPO is competitive across parameter settings.}
    \label{fig:all_experiments}
\end{figure*}

\begin{figure}[t]
    \centering
    \includegraphics[width=0.8\linewidth]{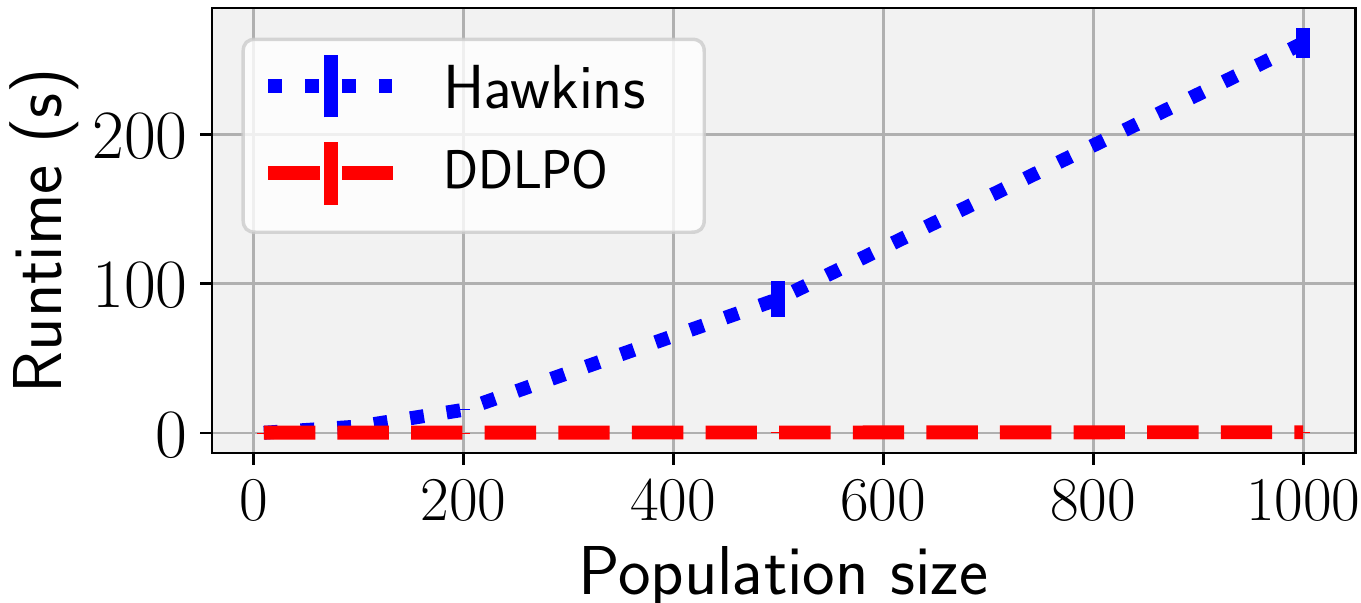}
 \caption{The poor scaling of query time of the Hawkins baseline compared to DDLPO, as discussed in Sec.~\ref{sec:experiments}, even for relatively small problem sizes ($N = 10, B = 2$).}
    \label{fig:hawkins_bad_runtime}
\end{figure}

\textbf{SIS Epidemic Model} is a discrete-state model in which arms represent distinct geographic regions and each member of an arm's population of size $N_{\textit{p}}$ is either (\textbf{S})usceptible to or (\textbf{I})nfected with an infectious disease. Such models have been the subject of increased interest following the COVID-19 pandemic \citep{hinch2021openabm,kerr2021covasim}, and will represent a large-state and multi-action experimental domain. In our model, the count of \textbf{S} members of the population is the state of each arm. Each arm's SIS model is defined by parameters $\kappa$, the average number of contacts per round, and $r_{\textit{infect}}$, the probability of infection given contact with an \textbf{I} member. Details on computing discrete state transition probabilities from these parameters are derived from \citet{yaesoubi2011generalized} and given in Appendix \ref{sec:appendix:domains:sis}. We introduce three intervention actions $\{a_0, a_1, a_2\}$ with costs $c=\{0, 1, 2\}$. Action $a_0$ represents no action, $a_1$ represents messaging about physical distancing (divides $\kappa$ by $a^{\textit{eff}}_1$), and $a_2$ represents distributing face masks (divides $r_{\textit{infect}}$ by $a^{\textit{eff}}_2$). We impose the following uncertainty intervals: $\kappa \in [1, 10]$, $r_{\textit{infect}} \in [0.5, 0.99]$, $a^{\textit{eff}}_{\{1,2\}} \in [1, 10]$.

\subsection{Performance of RR-DPO}
First, we evaluate the performance of the algorithms in uncertain environments. We compute the regret of an agent's pure strategy $\pi$ against a nature pure strategy $\omega$ as the difference in the average reward obtained by $\pi$ against $\omega$ and the average reward of the best strategy in the experiment against $\omega$. The average reward is the discounted sum of rewards over all arms for a horizon of length $10$, over $25$ simulations. In each setting, DO runs for $6$ epochs, using $100$ rollout steps and $100$ training epochs for each oracle. After completion, each baseline strategy is evaluated by querying the nature oracle for the best response against that strategy, then computing max regret against all $\omega$. The regret of RR-DPO is computed as the utility of the agent mixed strategy returned by the DO over the two-player regret game.

Fig.~\ref{fig:all_experiments}(a--f) shows RR-DPO incurs the lowest regret, beating the baselines in all domains. (a,b) shows results on the synthetic domain, demonstrating our approach can reduce regret by \char`\~$50\%$ against the benchmarks, across various values of $N$ and $B$. 
Moreover, as $B$ increases, the regret incurred may increase, since higher budget implies better reward potential for the optimal policy; however, the regret for RR-DPO remains small even as $B$ grows. 
Similarly, for the ARMMAN domain (c,d), a challenging domain adapted from a real-world problem, our algorithm performs consistently better than the baselines, achieving regret that is around $50\%$ lower than the best baselines. In the SIS domain (e--f), another real-world planning setting with a larger state space and multiple actions, our results are robust across parameter settings. Importantly, this holds even as we increase the state space from $S=100$ to $500$ (f), in which running the Hawkins baseline becomes prohibitively expensive. 

\emph{Finally, we run sensitivity analyses of the algorithms against $H$ and the size of the uncertainty sets} (Appendix Fig.~\ref{fig:um_and_h_sensitivity_analysis}). When $H$ varies from 10 to 100, RR-DPO maintains very low regret, while competitor regret as much as doubles, increasing RR-DPO's relative improvement as high as \char`\~60\%. Similar results are obtained when varying the uncertainty intervals between 0.25, 0.5 and 1.0 times their widths from the experiments in Fig.~\ref{fig:all_experiments}, with RR-DPO always dominating.

\subsection{Performance of DDLPO}
We also evaluate the performance of DDLPO, our novel DRL approach to find reward-maximizing policies for multi-action RMABs, which implements our agent oracle. We compare against \textbf{No Action} and \textbf{Random} baselines as well as the computationally intensive solution by Hawkins which computes the Lagrange policy, but which requires exact environment parameters and discrete states/actions. Hawkins upper bounds DDLPO for small discrete problems since it is exact whereas DDLPO learns the Lagrange policy from samples. Each experiment is a traditional reward-maximizing RMAB instantiated with a random sample of valid parameter settings for each seed.

Fig.~\ref{fig:all_experiments}(g--l) shows DDLPO achieves reward comparable to the Hawkins algorithm and significantly better than random, providing insight into the success of our RR-DPO approach which DDLPO enables, and showing promise for DDLPO as an algorithm of general interest. In the synthetic domain (g,h), DDLPO learns to act on the $33\%$ of arms who belong to category $W$. The mean reward of DDLPO almost matches that of Hawkins algorithm as $N$ scales with a commensurate budget (g). As we fix $N$ and vary the budget (h), the optimal policy accumulates more reward, and DDLPO almost equals the optimal. We observe similar results on the ARMMAN domain (i,j), a challenging real-world health problem. On the SIS domain (k,l), the strong performance of DDLPO holds in a multi-action setting even as we increase the number of states from 50 to 500 (l). 

Moreover, DDLPO beats Hawkins computationally: in Fig.~\ref{fig:hawkins_bad_runtime}, a single rollout ($10$ rounds) of Hawkins takes \char`\~$100$ seconds when there are $500$ states, scaling quadratically in general. This demonstrates that it would be prohibitive to run Hawkins in the loop of RR-DPO, since agent policies are evaluated thousands of times to compute the regret matrices. For just $25$ simulations, computation would take \char`\~$42$ minutes to evaluate a single cell in the regret matrix, which has $|\Pi| \times |\Omega|$ total cells. 



\section{Conclusion}
We address a key limitation blocking RMABs from many real-world settings: that arm dynamics are not known precisely. To plan safe, effective policies, robust approaches accounting for uncertainty are essential, which we give in RR-DPO, enabled by DDLPO, a novel deep-RL algorithm for RMABs of general interest. We hope our contributions bring us closer to deploying RMABs for real-world impact.

\begin{contributions} 
J.A.K.~conceived and implemented algorithmic ideas, wrote code, designed and ran experiments, wrote proofs, created figures, and led writing the paper. L.X.~contributed algorithmic ideas, wrote code, wrote proofs, and contributed to writing the paper. A.B.~contributed algorithmic ideas and contributed to writing the paper. M.T.~contributed guidance on the direction of the paper, contributed algorithmic ideas, and contributed to writing the paper.
\end{contributions}

\begin{acknowledgements}
This work was supported in part by the Army Research Office by Multidisciplinary University Research Initiative (MURI) grant number W911NF1810208. J.A.K. was supported by an NSF Graduate Research Fellowship under grant DGE1745303. A.B. was supported by the Harvard Center for Research on Computation and Society. Thank you to Andrew Perrault for feedback on an earlier draft.
\end{acknowledgements}

\nocite{NEURIPS2019_9015}
\nocite{gurobi}

\bibliography{killian_313}

\appendix
\clearpage

\renewcommand\thesection{\Alph{section}}
\renewcommand{\thefigure}{A\arabic{figure}}
\renewcommand{\thetable}{A\arabic{table}}
\renewcommand{\thealgorithm}{A\arabic{algorithm}}
\setcounter{figure}{0}
\setcounter{table}{0}
\setcounter{algorithm}{0}
\setcounter{proposition}{0}

\section{Proofs}
\label{sec:appendix:proofs}

\subsection{Proof of Proposition~\ref{thm:lambda_update}}
\begin{proposition}
To learn the value $\lambda$ that minimizes Eq.~\ref{eq:decoupled_value_func} given a state $\bm{s}$, the $\lambda$-network, parameterized by $\Lambda$, should be updated with the following gradient rule:
\begin{equation}
\begin{aligned}
    \Lambda_t = \Lambda_{t-1} - \alpha \left( \frac{B}{1-\beta} + \sum_{n=1}^{N}D_n(s_n, \lambda_{t-1}(\bm{s})) \right) 
\end{aligned}
\end{equation}
where $\alpha$ is the learning rate and $D_n(s_n, \lambda)$ is the negative of the expected $\beta$-discounted sum of action costs for arm $n$ starting at state $s_n$ under the optimal policy for arm $n$ for a given value of $\lambda$.
\end{proposition}

\begin{proof}
The gradient update rule is derived by taking the gradient of Eq.~\ref{eq:decoupled_value_func} with respect to $\lambda$, which has two main terms, $\lambda B / (1-\beta)$, and the sum over $Q_n$, the Q-functions with respect to $\lambda$. Looking more closely at $Q_n$, the only terms which are a function of $\lambda$ are the costs of actions taken by the policy that $Q_n$ implies, i.e., terms $-\lambda c_j$. Thus, the gradient of $Q_n$ is the negative expected discounted sum of costs taken by the optimal policy at the given value of $\lambda$, i.e., $\frac{dQ_n}{d\lambda} = -\mathbb{E}[\sum_{t=0}^{H} \beta^t c_{n,t}]$, where $c_{n,t}$ is the cost of the action taken on arm $n$ in round $t$.
\end{proof}

\subsection{Proof of Proposition~\ref{thm:lambda_convergence}}
\begin{proposition}
Given arm policies corresponding to optimal $Q$-functions, 
Prop.~\ref{thm:lambda_update} will lead $\Lambda$ to converge to the optimal as the number of training epochs and $K\xrightarrow[]{}\infty$.
\end{proposition}

\begin{proof}
Eq.~\ref{eq:decoupled_value_func} is convex in $\lambda$, which follows from definition of $Q_n$, i.e., the max over piece-wise linear functions of $\lambda$ is also a convex function in $\lambda$.  Thus the learning task of $\Lambda$ is also convex. Therefore, all that is required for asymptotic convergence of $\Lambda$ is that (1) the gradients we estimate via Prop.~\ref{thm:lambda_update} are accurate, and that (2) all inputs, i.e., all states $\bm{s}$, are seen infinitely often in the limit. (1) is achieved by the assumption that optimal $Q$-functions are given, an analytic condition that is achieved in practice by allowing the arm-networks to train for a reasonable number of rounds under a given output of the $\lambda$-network, before updating $\Lambda$. Specifically, given optimal Q-functions and their corresponding optimal policies, the sampled sums of spent budget from those optimal policies represent an unbiased estimator of each $D_n$. Note, though that to be an \textit{unbiased} estimator, this relies on not imposing the budget constraint \textit{at training time}, a procedure we carry out in practice.\footnote{It is critical to note that at test time, \textit{we always impose the budget constraint} -- i.e., all of our methods solve the original constrained RMAB problem -- they only use the Lagrangian relaxation as a tool to find good policies to the original constrained problem.} Thus (1) is achieved. (2) is achieved by following a training procedure that uniformly randomly samples start states $\bm{s}$ for each round of training until convergence. Thus the proposition is established.
\end{proof}

\subsection{Proof of Proposition~\ref{thm:rrdpo_convergence}}
\begin{proposition}
RR-DPO converges in a finite number of steps to the minimax regret-optimal policy.
\end{proposition}

\begin{proof}
A common strategy for establishing optimal convergence of the double oracle is to show that the pure strategy sets of both players can be exhausted. We can achieve this in our setting under the conditions (1) that each player has a finite strategy set, i.e., is possible to be exhausted and (2) that each oracle gives an optimal best response. Since the agent pure strategy set is already finite, we can achieve (1) by discretizing the nature oracle---in effect by rounding the outputs of the policy network. For (2), for analytical purposes, we make the common assumption that our oracles internally converge to their optimal values, i.e., in our case, the arm-networks and $\lambda$-network converge optimally. However, since our networks learn the Lagrange-relaxed version of the problem, some additional tools are needed. Speficially, we must identify conditions in which DDLPO-Act gives policies which approach $\pi^*_\omega$.
This can be achieved in the binary-action setting with $\alpha=$ `Whittle', which uses a binary search procedure to identify a value of $\lambda$ such that exactly $B$ arms have $Q_n(a=1,\lambda) > Q_n(a=0,\lambda)$, then acting on those arms. This procedure is equivalent to the Whittle index policy, which is asymptotically optimal for binary-action RMABs \citep{weber1990index}.
\end{proof}

\subsection{Proof of Proposition~\ref{thm:regret}}
\begin{proposition}
In the Robust RMAB problem with interval uncertainty, the max regret of a reward-maximizing policy can be arbitrarily large compared to a minimax regret-optimal policy.
\end{proposition}

\begin{proof}
Consider a binary-action RMAB problem with two arms A and~B. Let the reward from each arm be $R$ when the arm is in a \textit{good} state and $0$ in a \textit{bad} state. Our problem is to plan the best action with a budget of $1$ and horizon of $1$. Supposing the initial state is \textit{bad} for each arm, the transition probabilities for the transition matrix for each arm~$n$ is 
\begin{footnotesize}
$\begin{bmatrix} 1 & 0 \\ 1 - p_n & p_n \end{bmatrix}$\end{footnotesize}
where the uncertain variable~$p_n$ is constrained to be within $p_A, p_B \in [0, 1]$. Each value in the matrix corresponds to the probability of an arm at state \textit{bad} transitioning to \textit{bad} (column 1) or \textit{good} (column 2) if we take the \textit{passive} (row 1) or \textit{active} action (row 2).

To compute a reward-maximizing policy that does not consider robustness to uncertainty, we must optimize for one instantiation of the uncertainty set, which requires making one of three assumptions.
\begin{itemize}
    \item \textit{Case~1:} If we assume $p_A = p_B$, then an optimal policy is to act with probability $a_A$ on arm A and $a_B$ on arm B as long as $a_A + a_B = 1$. W.l.o.g., suppose $a_A \geq a_B$; then nature would set $p_A = 0$ and $p_B = 1$, imposing regret at least $R/2$. 
    \item \textit{Case~2:} If $p_A > p_B$, then the optimal policy would be to always act on arm A with probability $a_A = 1$ and never act on B ($a_B = 0$). Nature would then set $p_A = 0$ and $p_B = 1$ to impose regret $R$. 
    \item \textit{Case~3:} If $p_A < p_B$, the case is symmetric to Case~2 and result in regret $R$. Clearly, max regret is minimized when our action is such that $a_A + a_B = 1$; in this setting, we learn this optimal policy only under Case~1. Following Case 2 or~3, the difference between our regret and the minimax regret is $R/2$, which grows arbitrarily higher as $R \to \infty$.
\end{itemize}

A slight modification to this problem renders Case~1 non-optimal. Let the reward be $R$ when arm A is in a \textit{good} state and $R-1$ for arm B, so the optimal policy learned under the assumption from Case~1 leads to $a_A = 1$ and $a_B = 0$. Then nature could respond with $p_A = 0$ and $p_B = 1$, yielding reward $0$ and regret $R-1$, while the minimax regret--optimal policy achieves a minimum reward of $(R-1)/2$ (by playing $a_A = 0.5$ and $a_B = 0.5$ where nature responds with $p_A = 0$ and $p_B = 1$). Thus, the gap again can grow arbitrarily high as $R \to \infty$ provided that $R > 1$. We therefore have that in all cases, any reward-maximizing policy can achieve arbitrarily bad performance in terms of regret.
\end{proof}

\section{DDLPO-Act subroutines }
\label{sec:appendix:knapsack}
Here we provide the integer program which implements \texttt{QKnapsack}, one of the action-selection procedures used in Alg.~\ref{alg:ddlpo-act} to take actions at test time. \texttt{QKnapsack} takes $\lambda$ and $Q_n(s,a,\lambda)$ from the learned $\lambda$-network and arm networks, respectively, and returns the combination of actions that maximizes the sum of Q-values over all arms, subject to the costs of each action $\mathcal{C}$ and the budget constraint $B$.

\begin{align}
    &\max_{X} \sum_{n=1}^{N}\sum_{j=1}^{|\mathcal{A}|}x_{nj}Q_n(s_n, a_{nj}, \lambda) \\
    &\text{s.t. }\sum_{i=n}^{N}\sum_{j=1}^{|\mathcal{A}|} x_{nj}c_j \le B \\
    &\sum_{j=1}^{|\mathcal{A}|} x_{nj} = 1 \hspace{3mm} \forall n \in 1...N \\
    &x_{nj} \in \{0,1\}
    \label{eq:knapsack}
\end{align}

In Alg.~\ref{alg:bina-search}, we give the procedure $\texttt{BinaSearch}$ which implements a binary search over the learned $Q(\lambda)$-values to find a charge $\lambda$ for which exactly $B$ arms prefer to act rather than not act. This mimics the Whittle index policy in binary-action settings.

\begin{algorithm}[t]
\caption{BinaSearch (for the Whittle Index Policy)}
\label{alg:bina-search}
\textbf{Input}: State $\bm{s}$, arm critic networks $\phi_1,\ldots,\phi_N$, budget $B$, tolerance $\epsilon$.
\begin{algorithmic}[1] 
\STATE $q_{nj} = \phi_n(s_n,a_{nj},\lambda=0)$ \hspace{1mm} $\forall n \in [N], \forall j \in [|\mathcal{A}|]$
\STATE $lb = 0$
\STATE $ub = \max_{n\in [N], j \in [|\mathcal{A}|]}{\{q_{nj}}\}$
\color{black}\WHILE{$ub - lb > \epsilon$}
\STATE $\lambda = \frac{ub+lb}{2}$
\STATE $q_{nj} = \phi_n(s_n,a_{nj},\lambda)$ \hspace{1mm} $\forall n \in [N], \forall j \in [|\mathcal{A}|]$
\IF{fewer than $B$ arms have $q_{n,j=1} > q_{n,j=0}$}
\STATE $ub=\lambda$ \COMMENT{Charging too much, decrease}
\ELSIF{more than $B$ arms have $q_{n,j=1} > q_{n,j=0}$}
\STATE $lb=\lambda$ \COMMENT{Can charge more, increase}
\ELSIF{exactly $B$ arms have $q_{n,j=1} > q_{n,j=0}$}
\STATE break
\ENDIF
\ENDWHILE
\STATE $\bm{a} = \bm{0}$
\STATE $a_n = 1$ where $q_{n,j=1} > q_{n,j=0}$
\IF{$ub - lb \le \epsilon$}
\STATE break ties randomly s.t. $||\bm{a}||_1=B$
\ENDIF
\STATE \textbf{return} $\bm{a}$
\end{algorithmic}
\end{algorithm}

\section{Experimental Domain Details}
\label{sec:appendix:domains}
\subsection{ARMMAN}
\label{sec:appendix:domains:armman}
The MDPs in the ARMMAN domain \citep{biswas2021learn} have three ordered states representing the level of engagement of the beneficiaries in the previous week. Rewards are better for lower states, i.e., $R(0)=1, R(1)=0.5, R(2)=0$. At each step, the beneficiary may only change by one level, e.g., low-to-medium or high-to-medium but not low-to-high. They also assume that beneficiaries follow one of three typical patterns, A, B, and C, resulting in three MDPs with different transition probabilities. 
There are two patterns of effects present that differentiate the beneficiary types. (1)~For each of the above types, the planner can only make a difference when the patient is in state 1. Type A responds very positively to interventions, but regresses to low reward states in absence. Type B has a similar but less amplified effect, and type C is likely to stay in state 1, but can be prevented from regressing to state 2 when an action is taken. (2) Further, types A and C have only a 10\% chance of staying in the high reward state, while type B has a 90\% chance of staying there.

We converted these patient types to robust versions where the transition probabilities are uncertain as follows:

\[T^i_{s=0}=
\begin{bmatrix}
    p^i_{000} & 1 - p^i_{000} & 0.0 \\
    p^i_{010} & 1 - p^i_{010}  & 0.0
\end{bmatrix},
\]
\[T^i_{s=1}=
\begin{bmatrix}
    0.0 & 1 - p^i_{102} & p^i_{102} \\
    p^i_{110} & 1 - p^i_{110} & 0.0
\end{bmatrix}, \
\]
\[T^i_{s=2}=
\begin{bmatrix}
    0.0 & 1 - p^i_{202} & p^i_{202} \\
    0.0 & 1 - p^i_{212} & p^i_{212}
\end{bmatrix},
\]
where $i$ indexes the type (i.e., A, B or C).
We then set each $p^i_{sas^\prime}$ to be in a range of width 0.5 centered on the entries from each of the A, B, C beneficiary types for $s\in\{1,2\}$. To add additional heterogeneity to the experiments, for $s=0$, we set the range to 1.0 so that any beneficiary type can be made to have some non-negligible chance of staying in the good state, rather than only type B beneficiaries. The full set of parameter ranges are given in the Table~\ref{table:armman_robust_params} below.

\begin{table}[ht]
    \centering
    \begin{tabular}{|c|c|c||c|c||c|c|}
    \toprule
       Param &  L &  U & L &  U & L &  U \\
    \midrule
          Type A &        &        &  Type B &          &  Type C &          \\
    \midrule
     $p^i_{000}$ &   .00 &   1 &    .00 &     1 &    .00 &     1 \\
     $p^i_{010}$ &   .00 &   1 &    .00 &     1 &    .00 &     1 \\
     $p^i_{102}$ &   .50 &   1 &    .35 &     .85 &    .35 &     .85 \\
     $p^i_{110}$ &   .50 &   1 &    .15 &     .65 &    .00 &     .50 \\
     $p^i_{202}$ &   .35 &   .85 &    .35 &     .85 &    .35 &     .85 \\
     $p^i_{212}$ &   .35 &   .85 &    .35 &     .85 &    .35 &     .85 \\
    \bottomrule
    \end{tabular}
    \caption{Upper (U) and lower (L) parameter ranges for the robust ARMMAN environment.}
    \label{table:armman_robust_params}
\end{table}

In all experiments, 20\% of arms were sampled from type A, 20\% from type B and 60\% for type C. 
To add additional heterogeneity, for each of the 50 random seeds we uniformly sample a sub-range contained within the ranges given in Table~\ref{table:armman_robust_params}. In the agent oracle experiments, for each of the 50 random seeds, since these require fully instantiated transition matrices, we uniformly sample each parameter value for each arm according to its type such that the values are contained in the ranges given in Table~\ref{table:armman_robust_params}.

\subsection{SIS  Epidemic Model}
\label{sec:appendix:domains:sis}
In this domain, each arm follows its own compartmental SIS epidemic model. Each arm's SIS model tracks whether each of $N_p$ members of a population is in a susceptible (S) or infectious (I) state. This can be tracked with $N_p$ states, since it can be computed how many people are in state I if only the number of people in state S and the population size $N_p$ is known.

To define a discrete SIS model, we instantiate the model given in \citet{yaesoubi2011generalized} section 4.1 with a $\Delta t$ of 1. We also augment the model to include action effects and rewards. Specifically, $R(N_S) = N_S/N_p$, where $N_S$ is the number of susceptible (non-infected) people. Further, there are three actions $\{a_0, a_1, a_2\}$ with costs $c=\{0, 1, 2\}$. Action $a_0$ represents no action, $a_1$ divides the contacts per day $\kappa$ ($\lambda$ in \citet{yaesoubi2011generalized}) by $a^{\textit{eff}}_1$, and $a_2$ divides the infectiousness $r_{\textit{infect}}$ ($r(t)$ in \citet{yaesoubi2011generalized}) by $a^{\textit{eff}}_2$. That is, taking action $a_1$ will \textit{reduce} the average number of contacts per day in a given arm, and taking action $a_2$ will reduce the probability of infection given contact in a given arm, thus reducing the expected number of people that will become infected in the next round. However, to make this a robust problem, the relative effect sizes of each action for each arm will not be known to the planner, nor will the $\kappa$ or $r_{\textit{infect}}$. We impose the following uncertainty intervals for all arms: $\kappa \in [1, 10]$, $r_{\textit{infect}} \in [0.5, 0.99]$, $a^{\textit{eff}}_1 \in [1, 10]$, and $a^{\textit{eff}}_2 \in [1, 10]$. 

In the robust double oracle experiments, to add additional heterogeneity, for each of the 50 random seeds we uniformly sample a sub-range contained within the ranges given above for each arm. In the agent oracle experiments, for each of the 50 random seeds, since these require fully instantiated transition matrices, we uniformly sample each parameter value for each arm such that the values are contained in the ranges given above.

\section{Hyperparameter Settings and Implementation Details}
\label{sec:appendix:hyperparams-and-details}
\textbf{Neural networks: }All neural networks in experiments are implemented using PyTorch 1.3.1 \citep{NEURIPS2019_9015} with 2 fully connected layers each with 16 units and tanh activation functions, and a final layer of appropriate size for the relevant output dimension with an identity activation function. The output of discrete actor networks (i.e., the policy network from the agent oracle, and the policy network of agent A in the nature oracle) pass through a categorical distribution from which actions are randomly sampled at training time, without a budget imposed. It is critical not to impose the budget at training time, so that the budget spent by the optimal policy under a given $\lambda$ will result in a meaningful gradient for updating the $\lambda$-network. The output of continuous actor networks (i.e., agent B in the nature oracle which selects environment parameter settings) instead are passed as the means of Gaussian distributions -- with the log standard deviations learned as individual parameters separate from the network -- from which continuous actions are sampled at training time. At test time, actions are sampled from both types of networks deterministically. For categorical distributions, we greedily select the highest probability actions. For Gaussian distributions, we act according to the means. All discount factors were set to 0.9. The remaining hyperparameters that were constant for all experiments for the agent and nature oracles are indicated in Table \ref{table:hyperparams}. For \textbf{Robust Double Oracle} experiments, all agent and nature oracles were run for 100 training epochs. For \textbf{Agent Oracle} experiments, DDLPO was run for 100 training epochs for the synthetic and ARMMAN domains and 200 epochs for the SIS domain.

\textbf{$\lambda$-network: } Critical to training the $\lambda$-network is cyclical control of the temperature parameter that weights the entropy term in the actor loss functions. Recall that the $\lambda$-network is only updated every \texttt{n\_subepochs}. In general, after each update to the $\lambda$-network, we want to encourage exploration so that actor networks explore the new part of the state space defined by updated predictions of $\lambda$. However, after \texttt{n\_subepochs} rounds, we will use the cost of the sampled actor policies as a gradient for updating the $\lambda$-network, and that gradient will only be accurate if the actor policy has converged to the optimal policies for the given $\lambda$ predictions. Therefore, we also want to have little or no exploration in the round before we update the $\lambda$-network. In general, we would also like the entropy of the policy network to reduce over time so that the actor networks and $\lambda$-networks eventually both converge.

To accomplish both of these tasks, the weight (temperature) of the entropy regularization term in the loss function of the actor network will decay/reset according to two processes. The first process will linearly decay the temperature from some positive, but time-decaying  starting value (see next process) $\tau_t$ immediately after each $\lambda$-network update, down to 0 after $\texttt{n\_subepochs}$. The second process will linearly decay the temperature from a maximum $\tau_0$ (\emph{start entropy coeff} in Table~\ref{table:hyperparams}) down to $\tau_{\min}$ (\emph{end entropy coeff} in Table~\ref{table:hyperparams}) by the end of training. 

We found that it also helps to train the actor network with no entropy and with the $\lambda$-network frozen for \emph{lambda freeze epochs} rounds before training is stopped (Table~\ref{table:hyperparams}).

\textbf{Double Oracle: }In all experiments in the main text, we initialize the agent strategy list with HO, HM, and HP, and the nature strategy list with pessimistic, mean, and optimistic nature strategies, then run RR-DPO for 6 iterations. This produces a set of 8 agent strategies, 8 nature strategies, a table where each entry represents the regret of each agent pure strategy (row) against each nature pure strategy (column), and an optimal mixed strategy over each set that represents a Nash equilibrium of the minimax regret game given in the table. The regret table is computed by first computing the returns of each agent/nature pure strategy combination, then subtracting the max value of each column from all entries in that column (i.e., the best agent strategy for a given nature strategy gets 0 regret). The regret of RR-DPO is reported as the expected utility corresponding to the Nash equilibrium of the regret game given by the table, once that regret table is normalized to account for the returns of baselines (next). 

After this main loop completes, we then compute the regret of the baselines by evaluating each baseline policy against each pure strategy in the nature strategy list. Then, we also run the nature oracle against each baseline policy to find a nature strategy that should maximize the regret of that baseline. The regret for each baseline is reported as the max regret against this new nature strategy, as well as all pure nature strategies from the main RR-DPO loop. 

\begin{table}[t]
    \centering
    \begin{tabular}{|l|r|}
    \toprule
                 \textbf{parameter} &   \textbf{value} \\
    \midrule
                     \textit{\textbf{agent}} &         \\
                clip ratio & 2.0e+00 \\
      lambda freeze epochs & 2.0e+01 \\
       start entropy coeff & 5.0e-01 \\
         end entropy coeff & 0.0e+00 \\
       actor learning rate & 2.0e-03 \\
      critic learning rate & 2.0e-03 \\
      lambda learning rate & 2.0e-03 \\
          trains per epoch & 2.0e+01 \\
      n\_subepochs & 4.0e+00 \\
                           &         \\
                    \textit{\textbf{nature}} &         \\
                clip ratio & 2.0e+00 \\
      lambda freeze epochs & 2.0e+01 \\
       start entropy coeff & 5.0e-01 \\
         end entropy coeff & 0.0e+00 \\
      actorA learning rate & 1.0e-03 \\
     criticA learning rate & 1.0e-03 \\
      actorB learning rate & 5.0e-03 \\
     criticB learning rate & 5.0e-03 \\
      lambda learning rate & 2.0e-03 \\
          trains per epoch & 2.0e+01 \\
      n\_subepochs & 4.0e+00 \\
      n\_sims & 2.5e+01 \\
    \bottomrule
    \end{tabular}
    \caption{Hyperparameter settings for agent and nature oracles for all experiments.}
    \vspace{-5mm}
    \label{table:hyperparams}
\end{table}

\textbf{Hawkins Baselines: }The Hawkins policies are implemented with gurobipy 9.1.2, a Python wrapper for Gurobi (9.0.3) \citep{gurobi} following the LP given in \citet{hawkins2003langrangian} equation 2.5 to compute $\lambda$ and $Q(s,a,\lambda)$ for each arm and the integer program in equation 2.12 to select actions.

\textbf{RLvMid Baseline: }We found that RLvMid found effective policies for the nature strategy it was trained against (as evidenced in Figure \ref{fig:all_experiments})(a-f), but that that learned policy could be brittle against other nature strategies. This is likely because different nature strategies produce different distributions of states, meaning RLvMid would fit policies well to states seen when planning against the mean nature strategy, but underfit its policies for states seen more often in different distributions. 
\balance However, the lone RLvMid baseline policy can somewhat correct for this effect by training an ensemble of policies against slight perturbations of the mean nature strategy that adjust the parameter values output by nature by a small $\epsilon$. In all experiments we train 3 RLvMid policies against 3 random perturbations of the mean nature strategy, then report the regret of RLvMid as the minimum of the max regrets returned by any of the 3.

\begin{figure*}[t]
    \centering
    \includegraphics[width=0.9\linewidth]{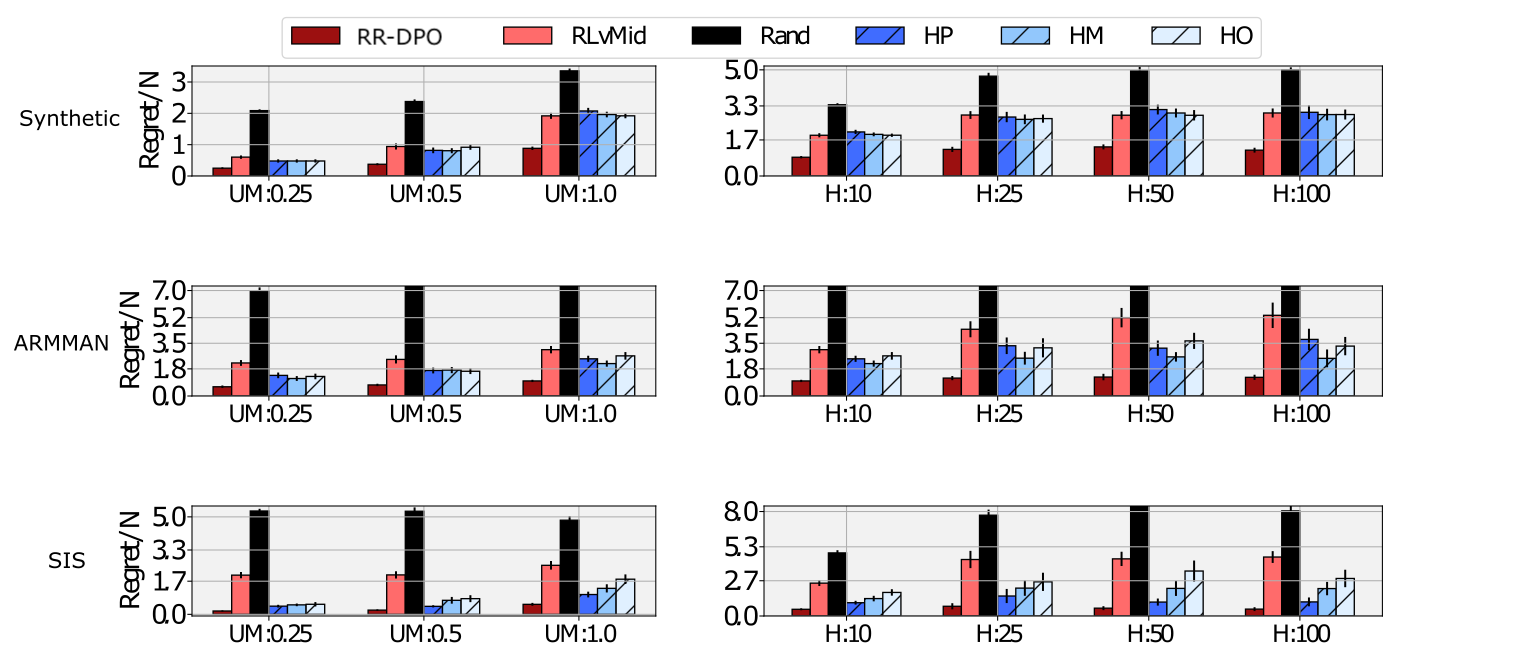}
    \caption{(Left column) varies the uncertainty intervals to be 0.25, 0.5 and 1.0 times their widths (UM = uncertainty multiplier). The gap between our robust RR-DPO method and non-robust methods becomes larger as the uncertainty interval increases, and our robust algorithm RR-DPO always provides the lowest regret policies. (Right column) varies the horizon H in {10, 25, 50, 100}. As expected, the gap between RR-DPO and the baselines either stays the same, or increases as H is increased, further demonstrating the robustness of our algorithm to various parameters.}
    \label{fig:um_and_h_sensitivity_analysis}
\end{figure*}


\end{document}